\documentclass{article}

\usepackage{microtype}
\usepackage{graphicx}
\usepackage{subcaption}
\usepackage{booktabs} 

\usepackage{hyperref}

\usepackage[preprint]{icml2026}

\usepackage{microtype}      
\usepackage{xcolor}         
\usepackage{graphicx}
\usepackage{wrapfig}
\usepackage{multirow}
\usepackage{bm}

\usepackage{amsmath}
\usepackage{amssymb}
\usepackage{mathtools}
\usepackage{amsthm}

\usepackage{algorithmic}

\usepackage[capitalize,noabbrev]{cleveref}

\theoremstyle{plain}
\newtheorem{theorem}{Theorem}[section]

\newtheorem{lemma}[theorem]{Lemma}

\theoremstyle{definition}
\newtheorem{definition}[theorem]{Definition}
\newtheorem{assumption}[theorem]{Assumption}
\theoremstyle{remark}

\usepackage[textsize=tiny]{todonotes}

\definecolor{myColor}{rgb}{0,0,0}

\begin{document}
\icmltitlerunning{Manifold-Constrained Energy-Based Transition Models for Offline Reinforcement Learning}

\twocolumn[

\icmltitle{Manifold-Constrained Energy-Based Transition Models for Offline Reinforcement Learning}


\icmlsetsymbol{equal}{*}

\begin{icmlauthorlist}
\icmlauthor{Zeyu Fang}{yyy}
\icmlauthor{Zuyuan Zhang}{yyy}
\icmlauthor{Mahdi Imani}{xxx}
\icmlauthor{Tian Lan}{yyy}
\end{icmlauthorlist}

\icmlaffiliation{yyy}{Department of Electrical and Computer Engineering, George Washington University}
\icmlaffiliation{xxx}{Department of Electrical and Computer Engineering, Northeastern University}

\icmlcorrespondingauthor{Zeyu Fang}{joey.fang@gwu.edu}

\icmlkeywords{Offline reinforcement learning, manifold learning, xxx}

\vskip 0.3in
]



\printAffiliationsAndNotice{}  
\begin{abstract}
Model-based offline reinforcement learning is brittle under distribution shift: policy improvement drives rollouts into state--action regions weakly supported by the dataset, where compounding model error yields severe value overestimation. We propose \emph{Manifold-Constrained Energy-based Transition Models} (MC-ETM), which train conditional energy-based transition models using a \emph{manifold projection--diffusion} negative sampler. MC-ETM learns a latent manifold of next states and generates near-manifold hard negatives by perturbing latent codes and running Langevin dynamics in latent space with the learned conditional energy, sharpening the energy landscape around the dataset support and improving sensitivity to subtle out-of-distribution deviations. For policy optimization, the learned energy provides a single reliability signal: rollouts are truncated when the minimum energy over sampled next states exceeds a threshold, and Bellman backups are stabilized via pessimistic penalties based on \emph{Q}-value-level dispersion across energy-guided samples. We formalize MC-ETM through a hybrid pessimistic MDP formulation and derive a conservative performance bound separating in-support evaluation error from truncation risk. Empirically, MC-ETM improves multi-step dynamics fidelity and yields higher normalized returns on standard offline control benchmarks, particularly under irregular dynamics and sparse data coverage.
\end{abstract}

\section{Introduction}
Offline reinforcement learning (RL) learns a policy solely from a fixed dataset collected by one or more behavior policies, without further interaction with the environment~\cite{prudencio2023survey}. A fundamental challenge is \emph{distribution shift}: as policy optimization progresses, the learned policy queries state--action pairs that are weakly supported by the dataset, where value estimation is poorly constrained and can become severely over-optimistic, leading to unstable training or performance collapse~\cite{fujimoto2019off}.

A common response is to suppress out-of-distribution (OOD) actions or constrain the learned policy toward the behavior distribution~\cite{fujimoto2021minimalist,yang2022rorl,mao2024offlinereinforcementlearningood}. While effective at limiting extrapolation, such constraints can be overly conservative when the dataset is suboptimal. Model-based offline RL instead learns a transition model from trajectories and performs policy evaluation and improvement through synthetic rollouts~\cite{yu2020mopo,moerland2023model,sun2023model}. However, this shifts the difficulty to \emph{robust transition modeling under shift}: once rollouts approach the boundary of dataset support, even small transition errors can compound rapidly and drive value targets toward unrealistic, high-return trajectories~\cite{janner2019trust}.

This brittleness is amplified in realistic environments for two structural reasons. First, many systems exhibit \emph{multi-modal or discontinuous} transitions (e.g., mode switching, bifurcations, contact-like effects), where regression-style forward models trained with mean squared error average incompatible outcomes and yield implausible predictions~\cite{pfrommer2021contactnets, allen2023graph}. Second, although observations may be high-dimensional, valid next states often concentrate near low-intrinsic-dimensional structures; perturbations that are small in ambient space can correspond to large departures from the transition support, destabilizing long-horizon evaluation even near the data boundary~\cite{chiappa2023latent, rafailov2021offline}.

We propose \emph{Manifold-Constrained Energy-based Transition Models} (MC-ETM) for robust offline model-based RL. MC-ETM trains conditional energy-based transition models with a geometry-aware \emph{manifold projection--diffusion} (MPD) mechanism~\cite{yoon2023energy}: it learns a latent manifold of next states and generates near-manifold hard negatives by perturbing latent codes of observed transitions and running Langevin dynamics under the learned energy. Concentrating negatives near the dataset support while enforcing conditional inconsistency forces the energy to form sharper contours around valid transitions, improving sensitivity to small but consequential OOD deviations.

MC-ETM then uses the learned conditional energy as a single reliability signal for offline policy optimization. We truncate synthetic rollouts when the inferred minimum energy over sampled next states exceeds a threshold, and we stabilize Bellman backups using pessimistic penalties based on \emph{$Q$-value-level} dispersion across energy-guided samples. We formalize the resulting algorithm as policy optimization in a hybrid pessimistic Markov Decision Process (MDP) and derive a conservative performance bound that separates in-support evaluation error from truncation risk.

Our main contributions are:
\begin{itemize}
    \item \textbf{Geometry-aware energy-based transition learning:} We introduce MC-ETM with MPD, which generates near-manifold hard negatives via latent perturbation and Langevin dynamics to sharpen the conditional energy landscape near the dataset support.
    \item \textbf{Reliability-aware offline policy optimization:} We use the learned energy to control rollout reliability through minimum-energy truncation and to stabilize learning via pessimistic $Q$-value penalties computed from energy-guided samples.
    \item \textbf{Theory and evidence:} We provide a hybrid pessimistic MDP formulation with a conservative bound separating in-support evaluation error from truncation risks, and demonstrate improved dynamics fidelity and higher normalized returns on standard offline control benchmarks (including D4RL MuJoCo), particularly under irregular dynamics and sparse data coverage.
\end{itemize}

\section{Related Work}
\label{sec:related_work}

\paragraph{Model-free offline RL.}
A large body of work mitigates distribution shift without an explicit dynamics model by constraining policy improvement toward the dataset support via behavior regularization or conservative value estimation (e.g., BCQ/BEAR/CQL/IQL/TD3+BC/EDAC and related variants)~\cite{fujimoto2019off, kumar2019stabilizing,  kumar2020conservative,kostrikov2021offline,fujimoto2021minimalist, an2021uncertainty, bai2022pessimistic, zhang2025learning}. These methods improve robustness by limiting extrapolation on OOD samples and value targets, but they can become overly conservative when the dataset is suboptimal and improvement requires generalization.

\paragraph{Model-based offline RL.}
Model-based offline RL learns a transition model from data and performs policy evaluation and improvement via synthetic rollouts, typically coupled with pessimism to reduce exploitation of model error under rollout shift (e.g., MOPO/MOReL/COMBO/MOBILE and related frameworks)~\cite{yu2020mopo,kidambi2020morel,yu2021combo,rigter2022rambo, sun2023model, zhang2024modeling, yu2025look}. Performance hinges on rollout reliability near the dataset boundary, where small transition errors compound and common uncertainty surrogates (often next-state disagreement/variance) can be misaligned with decision-relevant value overestimation in high-dimensional settings.

\paragraph{Energy-based transition modeling.}
Energy-based transition models (ETMs) represent conditional dynamics via a scalar energy over candidates and perform inference by sampling, capturing multi-modal or discontinuous transitions without collapsing to a single averaged prediction~\cite{du2019implicit, chen2024offline}. Because ETMs are trained contrastively, the negative-sample distribution is central: in high dimensions, noise-initialized negatives are often far from the dataset support, yielding ``easy'' contrasts that under-shape the energy landscape precisely in the near-support regime where subtle deviations govern rollout reliability.~\cite{pang2020learning, xiao2020vaebm}

\paragraph{Our approach.}
Our approach makes energy-based transition learning geometry-aware for offline model-based RL by introducing an MPD mechanism that learns a latent manifold of next states and generates near-manifold hard negatives via latent perturbations and Langevin dynamics under the learned conditional energy, sharpening the energy boundary near the dataset support.
The learned energy then serves as a unified reliability signal for offline policy optimization, enabling minimum-energy rollout truncation and value-level pessimistic regularization based on \emph{$Q$}-dispersion across energy-guided samples.

\section{Preliminaries}

\label{sec:preliminaries}

\paragraph{General RL}

The goal of reinforcement learning is to learn a policy that maximizes the expected return of an unknown specific Markov Decision Process (MDP), referred as the environment. A standard discounted MDP denoted as $\mathcal{M}$ can be defined as a six-element tuple
\(
\mathcal{M} = (\mathcal{S}, \mathcal{A}, P, r, \gamma, \mu_0),
\)
where \(\mathcal{S}\) is the state space, \(\mathcal{A}\) is the action space,
\(P(s' \mid s,a)\) is the state transition probability distribution, in which $s, s' \in \mathcal{S}$ and $a \in \mathcal{A}$. \(r(s,a)\in\mathbb{R}\) is the reward function,
\(\gamma \in (0,1)\) is the discount factor, and $\mu_0$ is the initial distribution of states. For a given policy $\pi(a|s)$, it can iteratively interact with this MDP in discrete time steps and induce a trajectory
\((s_0,a_0,r_0,s_1,\dots)\).
Based on $\pi$, the value function $V^\pi(s) = \mathbb{E}_{(s_t, a_t)\sim\mathcal{M},\pi} \left[ \sum_{t = 0}^{\infty} \gamma^{t} r(s_{t}, a_{t}) \mid s_{0} = s \right]$ is defined as the expected discounted cumulative rewards, or expected returns of the trajectories under $\pi$ started from the current state. Similarly, the Q-function $Q^\pi(s,a)=\mathbb{E}_{(s_t, a_t)\sim\mathcal{M},\pi} \left[ \sum_{t=0}^{\infty} \gamma^{t} r(s_t, a_t) \mid s_0 = s, a_0 = a \right]$ is defined as the expected returns under $\pi$ given a specific state and action. Thus, the objective of RL can be mathematically defined as finding the optimal policy $\pi^*$:
\begin{equation}
    \pi^* = \arg\max_\pi  \mathbb{E}_{s_0\sim\mu_0}\left[V^\pi(s_0)\right].
\end{equation}

\paragraph{Offline MBRL}

In offline RL, the learner does not have access to the environment during training but is given a fixed dataset
\(\mathcal{D} = \{(s_i, a_i, r_i, s'_i)\}_{i=1}^N\)
 instead, which is collected by an unknown behavior policy \(\pi_\beta\). The dataset provides an empirical state-action distribution with a probability density function
\(p_{\mathcal{D}}(s,a)\), and we denote the set of states and actions in the offline dataset by \(\mathcal{S}_{\mathcal{D}}\subseteq\mathcal{S}\) and
\(\mathcal{A}_{\mathcal{D}}\subseteq\mathcal{A}\), which are also the supports of the visited states and actions.

A major difficulty is that an improved policy \(\pi\) may increase the probability of state-action pairs
\((s,a)\) that are rare or absent under \(p_{\mathcal{D}}(s,a)\). We refer to such pairs as OOD pairs whose values are not constrained by the dataset. Since most deep reinforcement learning algorithms adopt a parameterized neural network $Q_\theta$ to approximate the real Q-function, 
when these methods query Q-values on OOD inputs during policy improvement, the function approximation can easily produce arbitrarily optimistic estimates, resulting in training instability or even failure.

Model-based approaches alleviate this by learning an approximate dynamics model
\(\hat{P}(s'| s,a)\) and using it to simulate synthetic rollouts.
Given a base offline RL algorithm, the learned model provides additional rollouts
\(\{(\hat{s}_t,a_t,\hat{r}_t,\hat{s}_{t+1})_{t=0}^{h}\}\) from which the critic and policy can be further improved, in which $h$ is the maximal rollout length, or referred as horizon.
However, when rollouts wander into regions poorly supported by $\mathcal{S}_D$ and $\mathcal{A}_D$,
the prediction error compounds and can also lead to severe value overestimation. Existing algorithms address this issue mostly by using uncertainty estimation techniques, such as ensembles, to function as a penalty to the reward. On the other hand, discontinuous dynamics can also cause high uncertainties in the next state prediction, even in in-distribution pairs, which has not yet attracted much attention.

\paragraph{ETMs}


Energy-Based Models (EBMs) provide a flexible framework for modeling probability distributions by associating a scalar energy value with each data point. Formally, given an input $x \in \mathcal{X}$, an energy function $E_\theta(x): \mathcal{X} \rightarrow \mathbb{R}$ parameterized by $\theta$ defines the probability density function $p_\theta(x)$ via the Boltzmann distribution:
\begin{equation}
    p_\theta(x) = \frac{\exp(-E_\theta(x))}{Z_x(\theta)},
\end{equation}
where $\quad Z_x(\theta) = \int \exp(-E_\theta(x)) dx$ is the intractable partition function. EBMs are not restricted to a specific distribution family (e.g., Gaussians) and can capture complicated dependencies and multi-modal structures by learning an energy landscape where observed data points are assigned lower energy values and unlikely configurations are assigned higher energy. To implement this feature, EBMs are typically trained with Maximum Likelihood Estimation (MLE), which seeks to maximize the log-likelihood of the observed data in the following gradient form:
\begin{equation}
\label{eq:MLE}
    \nabla_\theta \log p_\theta(x) = -\nabla_\theta E_\theta(x) + \mathbb{E}_{x^- \sim p_\theta(x)}[\nabla_\theta E_\theta(x^-)],
\end{equation}
where the first term decreases the energy of positive samples (\emph{i.e.,} in-distribution data point) and the second term increases the energy of negative samples generated from the model distribution. Since exact sampling from $p_\theta(x)$ is intractable, \textit{Langevin MCMC} is widely adopted to approximate negative samples. It iteratively updates a sample $\tilde{x}$ using the gradient of the energy function and additive noise:
\begin{equation}
    \tilde{x}_{k+1} \leftarrow \tilde{x}_k - \epsilon \nabla_x E_\theta(\tilde{x}_k) + \sqrt{2\epsilon} \omega_k, \quad \omega_k \sim \mathcal{N}(0, I),
\end{equation}
where $\epsilon$ is the step size.

Recent works have extended EBMs to sequential decision-making tasks by formulating the environment dynamics as a conditional distribution. An ETM represents the transition probability $p(s'|s, a)$ implicitly through a conditional energy function $E_\theta(s, a, s')$:
\begin{equation}
    p_\theta(s'|s, a) = \frac{\exp(-E_\theta(s, a, s'))}{\int \exp(-E_\theta(s,a,s')) ds'}.
\end{equation}
This formulation allows the model to capture discontinuous or multi-modal dynamics that are challenging for standard forward models. To train such models, previous methods employ contrastive learning objectives, such as InfoNCE, which can be viewed as a robust variant of MLE. During training, negative next-state samples are generated via conditional Langevin dynamics starting from noise or perturbed data. 
For inference, we approximately sample from $p_\theta(s'|s,a)$ using conditional Langevin dynamics. When a single-point prediction is needed, we report the final Langevin iterate as a low-energy sample.

\section{Methodology}
\label{sec:method}

In this section, we present the proposed MC-ETM framework. We detail the manifold-constrained training process of the transition model, followed by our uncertainty-aware policy optimization strategy. Finally, we provide a theoretical analysis of the proposed method.

\subsection{Manifold-Constrained Energy-Based Modeling}
\label{subsec:mc_etm}

The previous method of training ETM employs InfoNCE as the loss function, with negative samples generated independently via the Langevin MCMC process initialized from random noise. 
However, in many real-world scenarios, the dataset $\mathcal{D} \subset \mathbb{R}^d$ is high-dimensional and sparse. It typically lies on a lower-dimensional manifold $\mathcal{M}$ with intrinsic dimension $d_m \ll d$, which is related to the ambient space through a non-linear embedding.
Therefore, negative samples derived from simple Gaussian noise, as in the traditional ETMs, often lie far from the realistic data manifold due to the curse of dimensionality. These ``easy'' negatives fail to force the energy function to learn tight decision boundaries around the true transition dynamics, leading to loose energy landscapes that are insensitive to subtle OOD deviations. This problem further exacerbates the issue of inefficient training, since it requires a large number of steps to converge.

To address this, we introduce the Manifold Projection-Diffusion (MPD) mechanism into transition modeling. 
We assume next-states concentrate near a low-dimensional manifold $\mathcal{M}\subset\mathcal{S}$.
We approximate it using an autoencoder with encoder $f_e:\mathcal{S}\to\mathcal{Z}$ and decoder
$f_d:\mathcal{Z}\to\mathcal{S}$, trained to minimize
$\mathcal{L}_{\mathrm{rec}}=\|s' - f_d(f_e(s'))\|_2^2$ over $s'$ in $\mathcal{D}$.



\textbf{Training with MPD Negatives.} 
Instead of adding noise directly in the state space, we generate negative samples by perturbing the latent representation of the ground-truth next state. Given $(s,a,s')$, MPD projects the ground-truth next-state to latent space: $z=f_e(s')$, 
then diffuse the encoded vector $z$ with Gaussian noise to generate perturbations of the data within the manifold: $\tilde{z} = z +\sigma \epsilon, \epsilon \sim \mathcal{N}(0, I)$,
where $\sigma$ is a noise scale parameter. 

The resulting $\tilde{z}$ will be used to sample the latent negative samples $z^-$ through Langevin MCMC (LMCMC) within the manifold first, then passed to the original state space to generate the visible negative samples $s^-$. Such a two-stage sampling process ensures that the energy function is trained with in-manifold and plausible samples, as well as random noise, thus further enhancing performance.


By defining the latent energy function 
$H_\theta(z; s,a) := E_\theta(s,a,f_d(z))$, the conditional latent energy function $\tilde{H}_\theta(z|\tilde{z};s,a)$ can be derived as:
\begin{equation}
\tilde H_\theta(z \mid \tilde z; s,a)
= H_\theta(z; s,a) + \frac{1}{2\sigma^2}\|z-\tilde z\|_2^2 .
\end{equation}
Similarly, the energy function conditioned on $\tilde{z}$ can be obtained as follows:
\begin{equation}
\tilde E_\theta(s' \mid \tilde z;s,a)
= E_\theta(s,a,s') + \frac{1}{2\sigma^2}\|f_e(s')-\tilde z\|_2^2 .
\end{equation}
Since the conditional energy function shares the same gradient over $\theta$ as the original energy function, we can learn the energy function by iterative gradient descent to maximize the likelihood of the conditional recovery, \emph{i.e.}, $\log p_\theta(s'|\tilde{z};s,a)$, while the gradient of $\log p_\theta(x|\tilde{z};s,a)$ can be obtained by applying equation \eqref{eq:MLE} in its conditional form.

Given $(s,a,s')\sim\mathcal{D}$ and
negatives $\{s_i'^-\}_{i=1}^{N_s}$, we train MC-ETM using an InfoNCE-style contrastive objective:
\begin{equation}
\label{eq:loss}
\mathcal{L}_{\mathrm{NCE}}
= -\log \frac{\exp\{-E_\theta(s,a,s')\}}
{\sum_{i=0}^{N_s}\exp\{-E_\theta(s,a,s_i^-)\}},
\end{equation}
in which we define $s_0^- = s'$, which is the positive example taken directly from the dataset. 

Finally, we summarize the training process in Algorithm \ref{alg:training}.

\textbf{Manifold-Guided Inference.}
During inference, we perform a two-stage energy minimization process to predict the next state while fixing the current state and action. The first stage is to optimize the energy in the manifold space:
\begin{equation}\nonumber
z_{k+1}= z_k - \alpha_z \nabla_z H_\theta(z_k; s,a) + \sqrt{2\alpha_z}\,\omega_k,
\end{equation}
where $\omega_k \sim \mathcal{N}(0, I)$ is the noise and $\alpha_z$ is the noise scale. $z_0$ is an initialized vector, either from the output of a pre-trained MLP-based transition model $f_e(\hat s'_{\mathrm{MLP}})$ or from Gaussian noise in latent space. The first stage will be performed iteratively $K_z$ times, then passed to the state space for another $K_s$ iterations starting with $s'_0 = f_d(z_{K_z})$:

\begin{equation}
    s'_{k+1} = s'_k - \alpha_s \nabla_{s'} E_\theta(s, a, s'_k) + \sqrt{2\alpha_s}\omega_k
\end{equation}

The motivation behind this manifold-guided two-stage inference is similar to that of the training, which accelerates the convergence of the optimal next state with fewer sampling steps in total.


\begin{algorithm}[tb]
   \caption{Training of MC-ETM}
   \label{alg:training}
\begin{algorithmic}
   \STATE {\bfseries Input:} Dataset $\mathcal{D}$, Autoencoder $(f_e, f_d)$, noise scale $\sigma$, number of samples $N_s$
   \STATE Train $(f_e, f_d)$ on $\{s' \mid (s,a,r,s') \in \mathcal{D}\}$.
   \WHILE{converged}
   \STATE Sample batch $(s, a, s')$ from $\mathcal{D}$.
   \STATE Compute $\tilde{z} = f_e(s') + \sigma \epsilon, \epsilon \sim \mathcal{N}(0, I)$.
   \STATE Generate $z^-$ from $\tilde{z}$ using LMCMC through $\tilde{H}_\theta(z|\tilde{z})$.
   \STATE Generate $\{s'^-_i\}_{i=1}^{N_s}$ from $f_d(z^-)$ using LMCMC through $\tilde{E}_\theta(s,a,s'|\tilde{z})$.
   \STATE Define $s'^-_0 = s'$
   \STATE Update $\theta$ by gradient descent on $\mathcal{L}_{\mathrm{NCE}}$ defined in Eq.\ref{eq:loss}.
   \ENDWHILE
\end{algorithmic}
\end{algorithm}

\subsection{Resolving Uncertainties with Energy-based Models}

To robustly handle distribution shifts and overestimation of OOD samples, we propose an ETM-based Uncertainty Resolution framework. The trained MC-ETM model is used in two distinct ways:
\begin{itemize}
    \item \textbf{OOD Detection and Truncation:} The final value of $E_{\theta}(s, a, s')$ after the inference iteration captures the highest possible density under the current state-action pairs. If the energy remains high, it indicates that the current state-action pair is OOD. The model rollout will be truncated at this point to prevent further exploration.
    \item \textbf{Uncertainty Penalization:}By utilizing ensembles and assigning different initial vectors and noise scales, the MC-ETM can produce various samples of the next states, which can be used to further estimate both the aleatoric and epistemic uncertainties, respectively.
\end{itemize}

Before performing a Bellman update, we evaluate $E_{\theta}(s, a, s')$ on the optimized $s'$. We define a virtual terminal signal $\nu(s,a,s')$ based on this energy:
\begin{equation}
\label{nu}
    \nu(s, a,s') = \mathbb{I}(E_{\theta}(s, a, s') > \delta)
\end{equation}
where $\delta$ is a energy threshold. If $\nu(s, a,s') = 1$, the current rollout is truncated to effectively prevent the policy from exploring high-risk OOD regions.

Unlike previous ETM-based frameworks that estimate uncertainty through the standard deviation of merely the states, our MC-ETM framework adopts $M$ ensembles $\{E_{\theta_i}\}_{i=1}^M$, with each sampling $N$ next states $\{s_{i,j}\}_{j=1}^N$ with noise in Langevin iterations. 
The target Q-value is first calculated by averaging over samples within the same ensemble model, then incorporating an energy-based penalty based on the standard deviation over the ensembled Q-values:
\begin{align}
\label{eq:target}
    & \mathcal{T}Q_\phi(s,a) = r - \lambda std\{\bar q^\pi_{\phi,\theta_i}\}_{i=1}^M + \gamma \frac{1}{M} \sum_{i=1}^M \left( \bar q^\pi_{\phi,\theta_i} \right), \nonumber \\
    &  \bar q^\pi_{\phi,\theta_i} = \frac{1}{N}\sum_{j=1}^N (1-\nu(s,a,s'_{i,j}))Q_\phi(s'_{i,j},\pi(s'_{i,j})). 
\end{align}
Here, the standard deviation of the Q-values serves as a penalty for transitions with high epistemic uncertainty estimated over ensembles. Although the exact reasons for high uncertainty can be comprehensive, including multi-modal dynamics, discontinuous transitions, or simply in-correct model predictions, it can be reflected directly by observing the abnormal fluctuation of the Q-values. On the other hand, averaging over $M*N$ samples smooths the Q-value function, especially for dynamics with inconsistent transitions or terminal conditions. Unlike the transition function, the smooth regression bias that occurs here is beneficial for the estimated Q-function, since its gradient, instead of its absolute values, is utilized for policy iteration. 

We summarize the policy optimization process with MC-ETM in Algorithm \ref{alg:policy}. Traditional MBRL methods like MOPO and previous ETM-based methods rely on merely the variance of a bootstrap ensemble of dynamics models $\text{std}(\{s'_i \sim P_{\theta_i}(s,a)\})$ to estimate uncertainty and penalize the reward. Such a mechanism considers the model error as the sole source of uncertainty. In contrast, our model includes all other factors implicitly in the standard deviations of Q-values and stops the exploration on OOD samples explicitly using the energy model, thus significantly enhancing the stability of policy training.

\begin{algorithm}[tb]
   \caption{Policy Optimization with MC-ETM}
   \label{alg:policy}
\begin{algorithmic}
   \STATE {\bfseries Input:} Pretrained MC-ETM $E_{\theta}$, Policy $\pi_\psi$, parameterized Q-function $Q_\phi$.
   \FOR{each training step}
   \STATE Rollout with MC-ETM ensembles $E_{\theta_i}$ and extra terminal condition $\nu$ (Eq. \eqref{nu})
   \FOR{batch $(s,a,r,\{s'_{i,j}\})$ in replay buffer}
   \STATE Compute penalized target $\mathcal{T}Q_\phi(s,a)$ (Eq. \eqref{eq:target}).
   \STATE Update Q-function by minimizing $(Q_\phi(s,a) - \mathcal{T}Q_\phi(s,a))^2$.
   \STATE Update Policy $\pi_\psi$ using any compatible offline RL algorithm (\emph{e.g.}, SAC).
   \ENDFOR
   \ENDFOR
\end{algorithmic}
\end{algorithm}

\subsection{Theoretical Analysis}
\label{sec:theory}

In this section, we provide a theoretical justification for MC-ETM, demonstrating that our energy-constrained transition model yields a tighter performance bound compared to previous methods.

The core insight of MC-ETM is that epistemic uncertainty can only be validly estimated and bounded by ensemble standard deviations near the data manifold, where the model inputs are supported by the offline dataset. However, in OOD regions with high estimated energy, the ensemble predictions can degenerate, resulting in loose or invalid uncertainty bounds. To address this, we construct a hybrid pessimistic operator that switches between soft uncertainty penalization and hard truncation, determined by whether the model inputs are close to the estimated manifold.

\begin{definition}
\label{def:1}
We define the \textit{Energy-Constrained OOD Set} $\mathcal{U}_\delta$ as the set of state-action pairs where the model's energy exceeds a threshold $\delta$:
\begin{equation}
    \mathcal{U}_\delta = \{ (s,a) \in \mathcal{S} \times \mathcal{A} \mid \min_{s'} E_\theta(s,a,s') > \delta \}.
\end{equation}
Therefore, the MC-ETM value iteration framework can be considered a hybrid pessimistic Bellman operator $\hat{\mathcal{T}}_{MC}^\pi$ that applies different strategies based on this set:
\begin{align}
\label{eq:13}
    & \hat{\mathcal{T}}_{MC}^\pi \hat{Q}(s,a) = r(s,a) + \gamma \mathbb{I}\left((s,a)\notin \mathcal{U}_\delta \right)\mathbb{E} \left[\hat Q(s',a')\right],
\end{align}
where $s'$ and $a'$ are sampled from the transition probability of the dynamics model $\hat P$ and the current policy $\pi$.
\end{definition}

Our main result bounds the performance gap between the optimal policy $\pi^*$ and the learned policy $\hat{\pi}$ by decomposing the error into \textit{Consistency Error} (within the manifold) and \textit{Truncation Error} (outside the manifold).

\begin{assumption}[Manifold-Constrained Uncertainty Consistency]
    \label{ass:consistency}
    For state-action pairs supported by the data manifold ($ (s,a) \notin \mathcal{U}_\delta $), the Model-Bellman Inconsistency $u(s,a)$ is a valid estimator of the true Bellman error with coefficient $\beta$:
    \begin{equation}
        |\hat{\mathcal{T}}^\pi \hat{Q}(s,a) - \mathcal{T}^\pi \hat{Q}(s,a)| \le \beta u(s,a), \quad \forall (s,a) \notin \mathcal{U}_\delta.
    \end{equation}
\end{assumption}
This assumption is strictly weaker than the global assumption required by MOBILE and other previous ensemble-based offline RL algorithms, as it only requires the ensemble to be calibrated within the low-energy regions where the model training objective is active.

\begin{theorem}[Energy-Constrained Performance Bound]
    \label{thm:bound}
    Let $J(\pi)$ denote the expected return of policy $\pi$ in the true MDP. Under Assumption \ref{ass:consistency}, the performance gap between the optimal policy $\pi^*$ and the policy $\hat{\pi}$ learned by MC-ETM is bounded by:
    \begin{align}
        & J(\pi^*) - J(\hat{\pi}) \le  2 \sum_{t=0}^{H} \gamma^t \mathbb{E}_{\pi^*} \left[ \beta u(s_t, a_t) \cdot \mathbb{I}((s_t, a_t) \notin \mathcal{U}_\delta) \right] \nonumber \\
        & + \frac{2 R_{\max}}{1-\gamma} \mathbb{P}_{\pi^*} \left( \exists t : (s_t, a_t) \in \mathcal{U}_\delta \right)
    \end{align}
\end{theorem}

\textit{Proof Sketch.} We analyze the value difference $V^{\pi^*} - V^{\hat{\pi}}$ by expanding the telescope sum of Bellman errors along the trajectory of $\pi^*$. We partition the trajectory space based on the first hitting time of the set $\mathcal{U}_\delta$.
\begin{enumerate}
    \item \textbf{Scenario I:} When the policy remains within the energy manifold, the Bellman error is bounded by the penalized operator. Since we assume $(s,a) \notin \mathcal{U}_\delta$, Assumption \ref{ass:consistency} holds, and the term $\beta u(s,a)$ effectively upper-bounds the simulation error. This corresponds to the first term as the \textit{Consistency Error}. 
    \item \textbf{Scenario II:} When the policy attempts to visit a high-energy state, the rollout is truncated via $\nu$. In the worst case, the value difference is bounded by the maximum possible cumulative reward. This corresponds to the second term as the \textit{Truncation Risk}. By explicitly truncating these paths, we prevent the ``Model Exploitation'' phenomenon where standard ensembles might erroneously predict high values in OOD regions.
\end{enumerate}

This bound unifies the previous theoretical advantages of hard truncation and soft penalization. Inside the trust region, our bound accounts for the residual model error via the Consistency Error term, providing a tighter performance guaranty. While in high-energy regions where $u(s,a)$ can be arbitrarily large or erratic, our bound replaces the loose global penalty with a bounded truncation risk, ensuring that the optimization landscape is not polluted by the high variance of OOD ensemble predictions. The energy function $E_\theta$ acts additionally as a gatekeeper, ensuring the penalization mechanism is applied only where it is statistically valid.

\section{Experiments}

\label{sec:experiments}





\label{sec:experiments}

In this section, we evaluate the performance of our proposed MC-ETM framework and compare it with other existing approaches. Our experiments are designed to answer the following questions:
\begin{enumerate}
    \item \textbf{Dynamics Modeling:} Can manifolds help MC-ETM capture discontinuous dynamics better and more accurately where previous methods may fail?
    \item \textbf{Offline RL Performance:} Does MC-ETM achieve higher normalized returns on standard offline RL benchmarks compared to baselines?
    \item \textbf{Uncertainty Penalization:} How effective is the energy-based uncertainty resolution mechanism, and does it improve training stability?
\end{enumerate}

\subsection{Modeling Discontinuous Dynamics}
\label{subsec:toy_exp}

To obtain a clear view of the limitations of existing methods and the advantages of MC-ETM, we first perform an experiment on a didactic transition function $f: (\mathcal{S,A}) \to \mathcal{S}$, in which the state space $\mathcal{S}$ and action space $\mathcal{A}$ are both one-dimensional with a value range $[-1, 1]$. The true transition dynamics are defined by a discontinuous piecewise function that contains
multiple sharp sectors, with its details listed in the Appendix. We collected a dataset of 10,000 transitions covering the state space.

In this low-dimensional example, we primarily illustrate why energy-based transition models outperform smooth regressors on discontinuities. We then construct a high-dimensional nonlinear embedding where near-manifold negative sampling matters and compare MC-ETM with standard ETM to analyze the effect of MPD.
The compared models include MLP with standard probabilistic neural networks trained using MSE loss and a conditional diffusion model generating $s'$ via denoising. We show the visualized results in Figure \ref{fig:3d} and Figure \ref{fig:2d}, respectively.

\begin{figure*}
    \centering
    \includegraphics[width=0.9\linewidth]{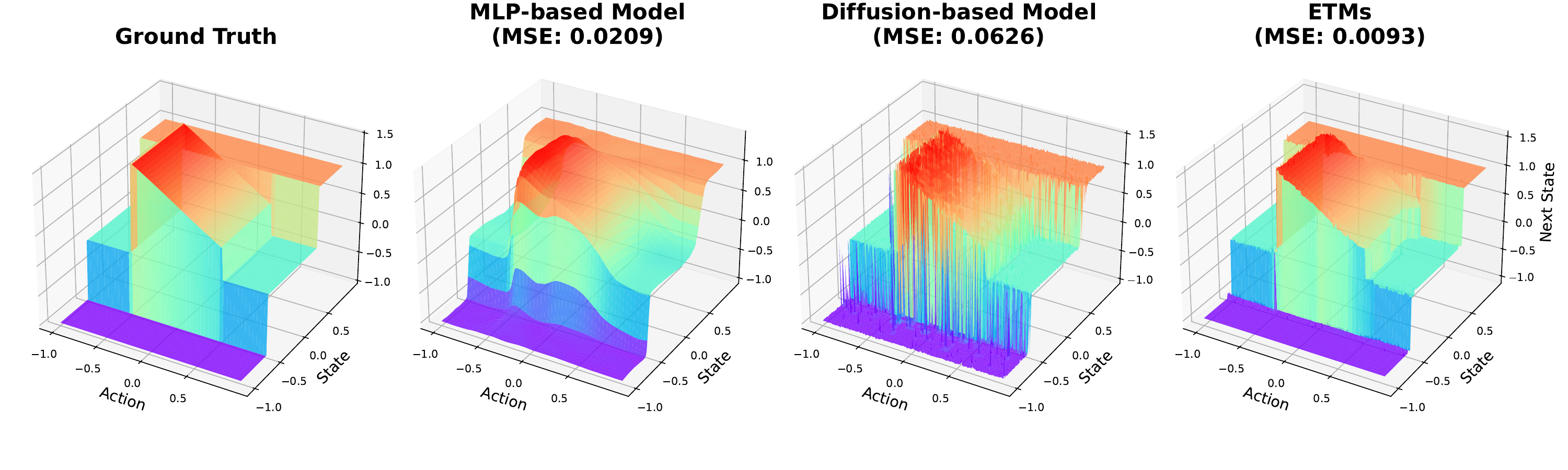}
    \caption{An illustrative example on fitting a discontinuous transition function.}
    \label{fig:3d}
\end{figure*}

\begin{figure}
    \centering
    \includegraphics[width=\linewidth]{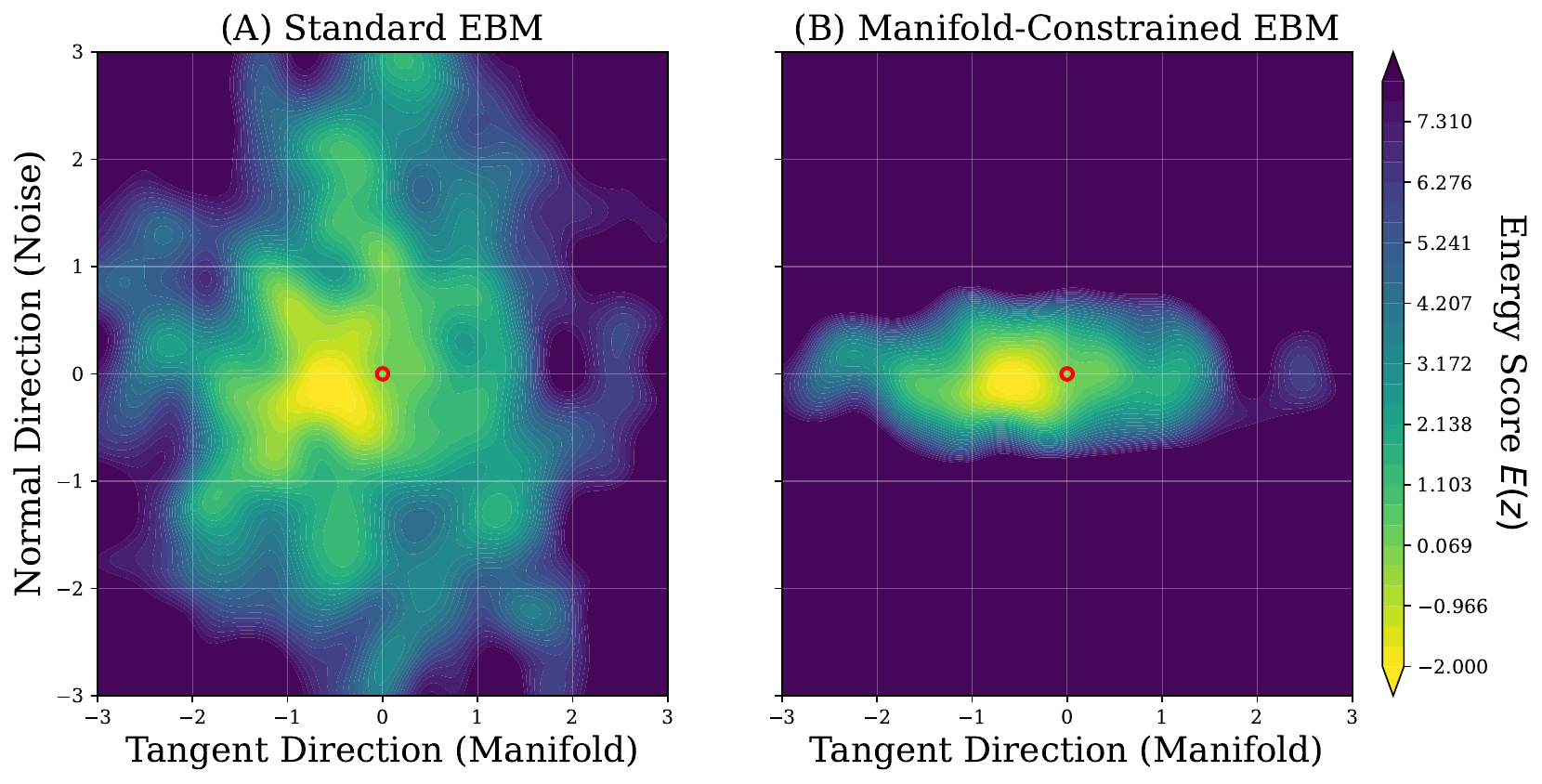}
    \caption{Conceptual visualization of energy landscapes on a 2D slice of the high-dimensional state space}
    \label{fig:2d}
\end{figure}
According to the prior result, the MLP models exhibit a strong smoothing bias, especially near the discontinuity $a= \pm 0.5$ and $s = \pm 0.5$, where the MLP predicts the average of the two modes, resulting in physically impossible states that never occur in the dataset, further distorting the predictions of in-distribution data. 
While the diffusion model avoids smoothing and captures the bi-modality better than MLPs, it suffers from high variance. Since diffusion models sample from the distribution rather than output the maximum likelihood estimate directly, the generated next states often deviate significantly from the precise function curve, leading to an even higher error. Furthermore, its training and inference time is much higher than that of the MLP due to the iterative denoising process.
The ETM, on the other hand, successfully captures the sharp discontinuity and achieves the lowest error, even in out-of-distribution points. In this example, there are no distinct differences between standard ETM and the MC-ETM, as manifold learning is omitted for one-dimensional state space.

Next, we embed the state and action spaces non-linearly into two high-dimensional spaces and then visualize the energy landscape learned by the two models on a 2D slice in Figure \ref{fig:2d}. The visualization slices the energy function $E(s, a, s')$ at a ground truth transition $s'$, spanned by a tangent vector along the data manifold and a normal vector orthogonal to the manifold. Darker regions indicate lower energy. While both models capture the general trend, the standard EBM exhibits a rugged, shallow landscape susceptible to off-manifold drift. The Manifold-Constrained EBM maintains a sharp energy barrier against invalid perturbations, despite local irregularities, recognizing that even small steps off the manifold are highly improbable.


\subsection{D4RL Benchmark Evaluation}
\label{subsec:d4rl_exp}

\begin{table*}[t]
\centering
\caption{Normalized return on D4RL MuJoCo locomotion
  tasks, averaged over 5 random seeds. The best results and the significant results are bolded. We use the letter r, m and e to abbreviate random, medium and expert, respectively.}
\label{tab:locomotion}
\resizebox{\textwidth}{!}{
\begin{tabular}{lcccccccc|c}
\toprule
\textbf{Task Name} & \textbf{CQL} & \textbf{TD3+BC} & \textbf{EDAC} & \textbf{MOPO} & \textbf{COMBO} &  \textbf{RAMBO} & \textbf{MOBILE} & \textbf{EMPO*} & \textbf{ETM} \\
\midrule
halfcheetah-r     & 31.3 & 11.0 & 28.4 & 38.5 & 38.8 & 39.5 & 39.3 & 14.3 & \textbf{40.7 ± 1.1} \\
hopper-r     & 5.3  & 8.5  & 25.3 & 31.7 & 17.9 & 25.4 & \textbf{31.9} & 30.9 & \textbf{31.8 ± 0.3} \\
walker-r     & 5.4  & 1.6  & 16.6 & 7.4  & 7.0  & 0.0  & {17.9} & 13.7 & \textbf{19.6 ± 1.3} \\
halfcheetah-m     & 46.9 & 48.3 & 65.9 & 73.0 & 54.2 & \textbf{77.9} & 74.6 & 21.2 & 76.9 ± 0.6 \\
hopper-m     & 61.9 & 59.3 & 101.6 & 62.8 & 97.2 & 87.0 & \textbf{106.6} & 32.9 & \textbf{107.0 ± 1.1} \\
walker-m     & 79.5 & 83.7 & \textbf{92.5} & 84.1 & 81.9 & 84.9 & 87.7 & 55.4 & \textbf{92.7 ± 0.7} \\
halfcheetah-m-r   & 45.3 & 44.6 & 61.3 & \textbf{72.1} & 55.1 & 68.7 & 71.7 & 8.4 & \textbf{72.4 ± 1.5} \\
hopper-m-r   & 86.3 & 60.9 & 101.0 & 103.5 & 89.5 & 99.5 & 103.9 & 34.9 & \textbf{104.8 ± 0.8} \\
walker-m-r   & 76.8 & 81.8 & 87.1 & 85.6 & 56.0 & \textbf{89.2} & \textbf{89.9} & 66.1 & \textbf{90.2 ± 1.3} \\
halfcheetah-m-e   & 95.0 & 90.7 & 106.3 & 90.8 & 90.0 & 95.4 & \textbf{108.2} & 28.1 & 105.2 ± 2.9 \\
hopper-m-e   & 96.9 & 98.0 & 110.7 & 81.6 & 111.1 & 88.2 & 112.6 & 41.8 & \textbf{113.8 ± 0.9} \\
walker-m-e   & 109.1 & 110.1 & {114.7} & 112.9 & 103.3 & 56.7 & {115.2} & 76.2 & \textbf{114.9 ± 1.8} \\
\midrule
\textbf{Average}      & 61.6 & 58.2 & 76.0 & 70.3 & 66.8 & 67.7 & 80.0 & 35.3 & \textbf{80.8} \\
\bottomrule
\end{tabular}
}
\end{table*}

We evaluate MC-ETM on the widely used continuous-control benchmark D4RL \cite{fu2020d4rl}, focusing on the Gym-Mujoco tasks (Hopper, Walker2d, HalfCheetah) under four dataset qualities: \textit{Random}, \textit{Medium}, \textit{Medium-Replay}, and \textit{Medium-Expert}. Following common practice, we report the average normalized return over $N=10$ evaluation rollouts per run and over $5$ random seeds. Error bars correspond to one standard deviation across seeds. The details about experimental configurations and additional experiments are included in the Appendix. 

We compare against a comprehensive suite of offline RL algorithms, including model-free algorithms CQL \cite{kumar2020conservative}, TD3+BC \cite{fujimoto2021minimalist}, EDAC \cite{an2021uncertainty}, and model-based algorithms
MOPO \cite{yu2020mopo}, COMBO \cite{yu2021combo}, RAMBO \cite{rigter2022rambo}, MOBILE~\cite{sun2023model} and the ETM-based algorithm EMPO \cite{chen2024offline}.
The offline policy iteration algorithm of MC-ETM is based on the SAC~\cite{haarnoja2018soft}, utilizing the authors' recommended architecture and optimization settings; however, the proposed method can be generalized to other standard offline policy learning patterns. The reported results of each method are from their published papers, except for EMPO, for which we reproduced the experiments using the official code and hyperparameters but found the reported results severely questionable.

\textbf{Overall Performance:} Table \ref{tab:locomotion} reports the normalized returns of all methods. Among them, MC-ETM achieves state-of-the-art performance and significantly outperforms traditional model-free and model-based algorithms. Particularly, the performance of EMPO drops significantly from the reported results and is lower than that of other SOTA algorithms. This phenomenon indicates that although ETM can reduce the prediction error, a proper uncertainty penalization mechanism is still crucial for robust policy training. Merely utilizing the standard deviation of predicted states in the explicit state space is often insufficient to estimate epistemic uncertainty or stabilize policy training. In contrast, our method shows consistent improvements, validating the efficacy of the manifold constraint and energy-based uncertainty penalization.

\subsection{Ablation Study}
\label{subsec:ablation}
To further analyze the performance of MC-ETM quantitatively across various metrics and investigate the contributions of each component and impacts of hyper-parameters, we conduct additional experiments in the ablation study and record detailed configurations and results in the Appendix.

\textbf{Prediction Error Analysis:}
To analyze the prediction accuracy of MC-ETM, we measured the prediction error across various environments with crossing datasets. Compared to MLPs, Diffusion Models, and standard ETMs, MC-ETM achieves the lowest MSE in almost all environments. Specifically, in out-of-distribution regions, MC-ETM still maintains a relatively low error, proving that it learns the underlying structure rather than merely memorizing the mean.

\textbf{OOD Detection and Penalty Correlation:}
To verify the effectiveness of our energy-based penalty, we computed the correlation between the penalty value and the actual model prediction error in the environment using the Pearson correlation coefficient metric. We also investigated the distribution of prediction error in both in-distribution and OOD samples, as determined by MC-ETM, verifying the effectiveness of energy-based OOD detection.

\textbf{Components and Hyperparameters}
We demonstrate the effectiveness of each component in MC-ETM by running experiments with the following configurations:
    \textbf{(i) Without Manifold Constraint (Standard ETM):} Removing the MPD and reverting to standard ETMs with energy-based uncertainty penalization intact.
    \textbf{(ii) Without OOD Detection:} remove the additional termination $\nu$ and only use the original termination condition. 
    \textbf{(iii) Without Energy-based Penalty:} Setting $\lambda = 0$ to deactivate the penalization mechanism in MC-ETM. 

Besides, we evaluate MC-ETM under various hyperparameter values, including the manifold dimension, the energy threshold for OOD detection $\delta$, and the penalty coefficient $\lambda$, to analyze their impacts on performance.


\section{Conclusion}
\label{sec:conclusion}

In this paper, we introduce MC-ETM, a unified energy model-based framework designed to address the critical challenges of offline MBRL: the overestimation of values on OOD samples and the difficulty of modeling non-smooth, discontinuous dynamics.
We replaced standard Gaussian noise negative sampling with MPD, forcing the energy function to learn tighter decision boundaries and effective OOD discrimination.
Besides, we design a new uncertainty penalization mechanism utilizing the energy value for OOD detection and the standard deviation of Q-values from multi-scale sampling for penalty estimation.
Empirically, we demonstrate that MC-ETM.produces less prediction error in both didactic environments and outperforms SOTA offline RL algorithms on D4RL benchmarks by (...\%). Finally, ablation studies confirmed the effectiveness of each component used in MC-ETM. 
Future work may explore extending this framework to high-dimensional visual observation spaces, where the manifold assumption is even more critical. Additionally, investigating the integration of MC-ETM with more advanced sampling strategies to further accelerate the inference process remains a promising direction.

\section*{Impact Statement}

This paper presents work whose goal is to advance the field of Machine
Learning. There are many potential societal consequences of our work, none of 
which we feel must be specifically highlighted here.

\bibliography{bibliography}
\bibliographystyle{icml2025}

\newpage
\appendix
\onecolumn

\section{Implementation Details}

We implement the MC-ETM based on the framework of OfflineRLKit \cite{sun2023model}, EMPO \cite{chen2024offline}, and MPDR \cite{yoon2023energy}. Our code is available at the anonymous link: \url{https://anonymous.4open.science/r/MC-ETM-Implementation-4CC9/}, which will be made public if the paper is accepted.

\begin{itemize}
    \item \textbf{Autoencoder ($f_e, f_d$):} The encoder and decoder consist of MLP architectures with 3 hidden layers and 64 hidden units per layer. We use ReLu for activation functions. The latent dimension $d_m$ is set to 5 for the hopper and 10 for the halfcheetah and walker2d tasks. The autoencoder is trained with the reconstruction error $l_r = \frac{1}{n}\sum_{i=1}^n( s'_i - f_d(f_e(s'_i)))^2$. 
    \item \textbf{Energy Function ($E_\theta$):} The energy function is parameterized as an MLP that takes the concatenation of $(s, a, s')$ as input and outputs a scalar energy value. It utilizes 4 hidden layers with 200 units and ReLu activations. The energy model is optimized with the InfoNCE loss defined in Eq.\eqref{eq:loss}. Following the configurations in EMPO~\cite{chen2024offline}, we add a gradient loss $\mathcal{L}_{g} = \sum_{i=0}^{N_s}\max(0, (\Vert \nabla_{s'}E_\theta(s,a,s_i^-) \Vert - G)^2)$ as a regularization term, in which $G$ is the gradient penalty margin as a hyperparameter.
    \item \textbf{Policy Iteration:} Following the base offline RL algorithm, we use SAC-style architectures for the actor $\pi_\psi$ and the twin-Q critics $Q_\phi$, with most hyperparameters associated with policy iteration following. Besides, we train another MLP-based reward model to predict the reward function based on $(s,a,s')$, while in practice $s'$ is the sample generated by the energy-based dynamics. Also in practice, we clip the next target Q to stay above 0, since in policy iteration we don't want the penalty to drop the Q-value extremely low to deteriorate the training stability.
\end{itemize}

All of the experiments in this paper are conducted on a server with an AMD EPYC 7513 32-Core Processor CPU and an NVIDIA RTX A6000 GPU. The training of a single energy model takes about 5 to 7 hours, and the inference time is about 5 to 6 times longer than that of MLP-based models. However, since the process of simulating rollout takes only about 10\% of the time cost in policy optimization, this time increase is acceptable. Besides, The usage of MPD doesn't see a significant increase in inference time compared to the standard ETM model.

\section{Hyperparameters}
We list the hyperparameters associated with energy-based transition model training in Table \ref{tab:hyper etm}, and those associated with policy optimization in Table \ref{tab:hyper policy}. Since we adopt SAC for policy optimization, most hyperparameters follow its standard implementations. Besides, two core hyperparameters in our proposed energy-based uncertainty penalization process, \emph{i.e.,} penalty coefficient $\lambda$ and the energy threshold $\delta$ are tuned over various environments and reported in Table ~\ref{tab:hyper env}.

\begin{table}[h]
    \centering
    \caption{Hyperparameters used in energy model training}
    \label{tab:hyper etm}
    \begin{tabular}{l c}
        \toprule
        \textbf{Hyperparameter} & \textbf{Value} \\
        \midrule
        Batch size              & 1024          \\
        Max epochs             & 100           \\
        Optimizer                         & Adam        \\
        Learning rate                     & 1e-3           \\
        Number of negative samples        & 20 \\
        Langevin inference step in manifold space          & 30          \\
        Langevin inference step in ambient space          & 20          \\
        Langevin inference noise          & 0.5         \\
        Langevin step size          & 1e-3         \\
        Langevin delta clip          & 0.5   \\
        Gradient penalty margin     & 5 \\

        \bottomrule
    \end{tabular}
\end{table}
\begin{table}[h]
    \centering
    \caption{Hyperparameters used in policy optimization}
    \label{tab:hyper policy}
    \begin{tabular}{l c}
        \toprule
        \textbf{Hyperparameter} & \textbf{Value} \\
        \midrule
        Total gradient steps              & 3M          \\
        Model ensemble number             & 5           \\
        Model inference sample number      & 10 \\
        Critic number                     & 2           \\
        Q network hidden layers                 & [256, 256]  \\
        policy network hidden layers            & [256, 256]  \\
        reward network hidden layers            & [256, 256]  \\
        Target Q smoothing coefficient    & 5e-3        \\
        Discount factor($\gamma$)                   & 0.99        \\
        Batch size                        & 256         \\
        Q learning rate                   & 3e-4        \\
        Actor learning rate               & 1e-4        \\
        Reward learning rate              & 1e-4        \\
        Optimizer                         & Adam        \\
        Ratio of real experiences         & 0.05        \\
        Langevin inference step in manifold space           & 30          \\
        Langevin inference step in ambient space           & 30          \\
        Rollout horizon & 5 \\
        Langevin step size                & 1e-3         \\
        Langevin inference noise          & 0.1         \\
        \bottomrule
    \end{tabular}
\end{table}

\begin{table}[h]
    \centering
    \caption{Placeholder Caption}
    \label{tab:hyper env}
    \begin{tabular}{l c c}
        \toprule
        \textbf{Task} & \textbf{Penalty Coefficient} & \textbf{Energy Threshold} \\
        \midrule
        hopper-random              & 0.5   & -2.5  \\
        hopper-medium              & 1.5   & -5  \\
        hopper-medium-replay       & 0.5   & -2.5  \\
        hopper-medium-expert       & 1.5   & 0  \\
        walker-random              & 2.5  & 5  \\
        walker-medium              & 1.0  & -5  \\
        walker-medium-replay       & 0.5  & 5  \\
        walker-medium-expert       & 1.0  & -10  \\
        halfcheetah-random         & 0.5  & -2.5  \\
        halfcheetah-medium         & 0.5  & 0  \\
        halfcheetah-medium-replay  & 0.5  & 2.5 \\
        halfcheetah-medium-expert  & 1.5  & -10  \\
        \bottomrule
    \end{tabular}
\end{table}

\section{Proofs of the main theorem}\label{app:theory}

\subsection{Preliminaries and Notation}\label{app:prelim}
We consider a discounted MDP $M=(\mathcal{S},\mathcal{A},P,r,\gamma,\rho_0)$ with $\gamma\in(0,1)$ and bounded rewards
$|r(s,a)|\le R_{\max}$. Let $V_{\max}\doteq \frac{R_{\max}}{1-\gamma}$.
For any stationary policy $\pi$, define the true Bellman operator
\begin{equation}
  (T^\pi Q)(s,a) \doteq r(s,a) + \gamma\,\mathbb{E}_{s'\sim P(\cdot|s,a),\,a'\sim\pi(\cdot|s')}\bigl[Q(s',a')\bigr],
\end{equation}
and let $Q^\pi$ denote the unique fixed point of $T^\pi$. Define the value function  $V^\pi(s)=\mathbb{E}_{a\sim\pi(\cdot|s)}[Q^\pi(s,a)]$.
We write $J(\pi)\doteq \mathbb{E}_{s_0\sim\rho_0}[V^\pi(s_0)]$ for the expected return in the true MDP.
MC-ETM uses an ensemble of learned dynamics models $\{\hat P_i\}_{i=1}^M$ and implicitly the averaged model
$\hat P \doteq \frac1M\sum_{i=1}^M \hat P_i$ for Bellman evaluation.
Let $\hat T^\pi$ denote the model Bellman operator under $\hat P$:
\begin{equation}
  (\hat T^\pi Q)(s,a) \doteq r(s,a) + \gamma\,\mathbb{E}_{s'\sim \hat P(\cdot|s,a),\,a'\sim\pi(\cdot|s')}\bigl[Q(s',a')\bigr].
\end{equation}

\paragraph{Energy-constrained OOD set and truncation.}
Following Definition~\ref{def:1}, we define the energy-constrained OOD set
\begin{equation}
  \mathcal{U}_\delta \doteq \Bigl\{(s,a)\in\mathcal{S}\times\mathcal{A}:\ \min_{s'}E_\theta(s,a,s')>\delta\Bigr\},\qquad
  \bar{\mathcal{U}}_\delta \doteq (\mathcal{S}\times\mathcal{A})\setminus\mathcal{U}_\delta .
  \label{eq:ood-set}
\end{equation}
MC-ETM additionally defines a virtual termination signal (Eq.~\eqref{nu})
\begin{equation}
  \nu(s,a,s') \doteq \mathbb{I}\!\left(E_\theta(s,a,s')>\delta\right).
  \label{eq:nu}
\end{equation}
Intuitively, $\mathcal{U}_\delta$ is the set of state-action pairs that are \emph{certifiably} OOD in the sense that even the
best-case next state has energy above threshold, while $\nu$ performs sample-wise truncation during rollouts.

\paragraph{Model-Bellman inconsistency / uncertainty quantifier.}
Let $u(s,a)$ denote an uncertainty score computed from the ensemble Bellman targets.
A concrete instance consistent with Eq.~\eqref{eq:target} is:
\begin{align}
  \bar q^{\pi}_{Q,\theta_i}(s,a)
  &\doteq \frac{1}{N}\sum_{j=1}^N \bigl(1-\nu(s,a,s'_{i,j})\bigr)\,Q\!\bigl(s'_{i,j},\pi(s'_{i,j})\bigr),
  \label{eq:qbar}\\
  u(s,a) &\doteq \mathrm{Std}_{i\in[M]}\!\Bigl[\bar q^{\pi}_{Q,\theta_i}(s,a)\Bigr],
  \label{eq:ubi}
\end{align}
where $s'_{i,j}$ are MC-ETM-generated next-state samples from the $i$-th model.
This specializes the \emph{Model-Bellman Inconsistency} idea of MOBILE~\cite{sun2023model} to the MC-ETM rollout mechanism, more precisely, the use of ensemble Bellman-target dispersion as an uncertainty quantifier.

\subsection{Hybrid pessimistic operator and assumption}\label{app:hybrid-op}
MC-ETM’s theoretical analysis can be expressed via a \emph{hybrid pessimistic Bellman operator} that combines
(i) soft penalization using $u(s,a)$ on $\bar{\mathcal{U}}_\delta$, and
(ii) hard truncation on $\mathcal{U}_\delta$.
Define, for any policy $\pi$ and any $Q$,
\begin{equation}
  (\hat T^{\pi}_{\mathrm{MC}} Q)(s,a)
  \doteq r(s,a) + \gamma\,\mathbb{I}\!\bigl((s,a)\in \bar{\mathcal{U}}_\delta\bigr)\,
  \mathbb{E}_{s'\sim \hat P(\cdot|s,a),\,a'\sim\pi(\cdot|s')}\bigl[Q(s',a')\bigr],
  \label{eq:hybrid-bellman}
\end{equation}
which matches Eq.~\eqref{eq:13} in the main text (hard truncation outside $\bar{\mathcal{U}}_\delta$).
We further define the \emph{penalized} operator used for pessimistic value estimation:
\begin{equation}
  (\hat T^{\pi}_{\mathrm{MC},\beta} Q)(s,a)
  \doteq r(s,a) - \beta\,u(s,a)\,\mathbb{I}\!\bigl((s,a)\in \bar{\mathcal{U}}_\delta\bigr)
  + \gamma\,\mathbb{I}\!\bigl((s,a)\in \bar{\mathcal{U}}_\delta\bigr)\,
  \mathbb{E}_{\hat P,\pi}\bigl[Q(s',a')\bigr].
  \label{eq:hybrid-penalized}
\end{equation}
Let $\hat Q^{\pi}_{\mathrm{MC},\beta}$ denote the fixed point of $\hat T^{\pi}_{\mathrm{MC},\beta}$
(and $\hat V^{\pi}_{\mathrm{MC},\beta}$ the induced value).

\paragraph{Assumption (Manifold-Constrained Uncertainty Consistency).}
Assumption~\ref{ass:consistency} states that, restricted to the low-energy set $\bar{\mathcal{U}}_\delta$, the uncertainty score
$u(s,a)$ upper-bounds the \emph{model Bellman error}:
\begin{equation}
  \bigl|(\hat T^\pi Q)(s,a) - (T^\pi Q)(s,a)\bigr|
  \le \beta\,u(s,a),\qquad \forall (s,a)\in \bar{\mathcal{U}}_\delta.
  \label{eq:assump-local}
\end{equation}
Compared to MOBILE’s global calibration requirement, Eq. \eqref{eq:assump-local} only needs to hold on the manifold/low-energy region.

\subsection{MOReL-style pessimistic MDP view}\label{app:pm}
We next connect Eq.~\eqref{eq:hybrid-bellman} to the pessimistic-MDP construction used in MOReL~\cite{kidambi2020morel}.
Define an augmented MDP $\hat M_{\mathrm{MC}}$ by adding an absorbing terminal state $\bot$ with $r(\bot,\cdot)=0$ and
$P(\bot|\bot,\cdot)=1$. Its transition kernel is
\begin{equation}
  \hat P_{\mathrm{MC}}(s'|s,a) \doteq
  \begin{cases}
    \hat P(s'|s,a), & (s,a)\in\bar{\mathcal{U}}_\delta,\ s'\in\mathcal{S},\\
    1, & (s,a)\in\mathcal{U}_\delta,\ s'=\bot,\\
    0, & (s,a)\in\mathcal{U}_\delta,\ s'\in\mathcal{S},
  \end{cases}
  \label{eq:pm-kernel}
\end{equation}
with reward $r_{\mathrm{MC}}(s,a)=r(s,a)$ for $s\in\mathcal{S}$ (and $0$ at $\bot$).
Then the standard Bellman operator for $\hat M_{\mathrm{MC}}$ coincides with \eqref{eq:hybrid-bellman}.
Thus, MC-ETM’s truncation is equivalently planning in a MOReL-style pessimistic MDP where entering $\mathcal{U}_\delta$
forces absorption (in MOReL this absorption is often paired with a large negative reward; here MC-ETM uses termination,
which we bound via the value range).

\subsection{Main technical lemmas}\label{app:lemmas}

\begin{lemma}[Pessimism on the manifold]\label{lem:pessimism-known}
Under Assumption~\eqref{eq:assump-local}, for any policy $\pi$, any bounded $Q$, and any $(s,a)\in\bar{\mathcal{U}}_\delta$,
\begin{equation}
  (\hat T^{\pi}_{\mathrm{MC},\beta} Q)(s,a) \le (T^\pi Q)(s,a).
  \label{eq:pessimism-known}
\end{equation}
\end{lemma}
\begin{proof}
Fix $(s,a)\in\bar{\mathcal{U}}_\delta$. By definition \eqref{eq:hybrid-penalized},
\[
(\hat T^{\pi}_{\mathrm{MC},\beta} Q)(s,a)
= r(s,a) + \gamma\,\mathbb{E}_{\hat P,\pi}[Q(s',a')] - \beta u(s,a).
\]
Also $(\hat T^\pi Q)(s,a)= r(s,a) + \gamma\,\mathbb{E}_{\hat P,\pi}[Q(s',a')]$ and $(T^\pi Q)(s,a)= r(s,a) + \gamma\,\mathbb{E}_{P,\pi}[Q(s',a')]$.
Rearranging Assumption~\eqref{eq:assump-local} gives
$(\hat T^\pi Q)(s,a) - \beta u(s,a) \le (T^\pi Q)(s,a)$, which is exactly \eqref{eq:pessimism-known}. 
\end{proof}

\begin{lemma}[Truncation loss bound via hitting time]\label{lem:truncate-risk}
Let $\tau_{\mathcal{U}}\doteq \inf\{t\ge 0:(s_t,a_t)\in\mathcal{U}_\delta\}$ be the first hitting time of $\mathcal{U}_\delta$
under policy $\pi$ in the \emph{true} MDP $M$ (with $\tau_{\mathcal{U}}=\infty$ if never hit).
Then for any initial state $s$,
\begin{equation}
  0 \le V^\pi(s) - \hat V^{\pi}_{\mathrm{MC}}(s) \le V_{\max}\,\mathbb{P}\!\left(\tau_{\mathcal{U}}<\infty \mid s_0=s\right),
  \label{eq:truncate-risk}
\end{equation}
where $\hat V^{\pi}_{\mathrm{MC}}$ is the value of $\pi$ in the truncated MDP $\hat M_{\mathrm{MC}}$ (without the soft penalty).
\end{lemma}
\begin{proof}
Couple the trajectories of $M$ and $\hat M_{\mathrm{MC}}$ so they agree up to time $\tau_{\mathcal{U}}-1$,
and in $\hat M_{\mathrm{MC}}$ the process transitions to $\bot$ at time $\tau_{\mathcal{U}}$.
On the event $\{\tau_{\mathcal{U}}=\infty\}$, the truncation never triggers and the two returns coincide.
On the event $\{\tau_{\mathcal{U}}<\infty\}$, the truncated return drops all rewards after $\tau_{\mathcal{U}}$,
whose absolute contribution is at most $V_{\max}$ by reward boundedness.
Taking expectations yields \eqref{eq:truncate-risk}. 
\end{proof}

\subsection{Proof of Theorem~4.3}\label{app:proof-main}
We now give a complete proof of the performance bound stated in Theorem~4.3.
For clarity, we state the theorem in the appendix notation.

\begin{theorem}[Energy-Constrained Performance Bound (restated)]\label{thm:main}
Assume $|r(s,a)|\le R_{\max}$ and Assumption~\eqref{eq:assump-local} holds on $\bar{\mathcal{U}}_\delta$.
Let $\hat\pi$ be a policy obtained by (approximately) optimizing the pessimistic model defined by \eqref{eq:hybrid-penalized},
and let $\pi^\star$ be an optimal policy in the true MDP.
Then, for any rollout horizon $H\in\mathbb{N}\cup\{\infty\}$,
\begin{equation}
  J(\pi^\star)-J(\hat\pi)
  \le
  2\sum_{t=0}^{H}\gamma^t\,\mathbb{E}_{\pi^\star}\!\Bigl[\beta\,u(s_t,a_t)\,\mathbb{I}\!\bigl((s_t,a_t)\in\bar{\mathcal{U}}_\delta\bigr)\Bigr]
  + \frac{2R_{\max}}{1-\gamma}\,
  \mathbb{P}_{\pi^\star}\!\Bigl(\exists t\le H:\ (s_t,a_t)\in \mathcal{U}_\delta\Bigr).
  \label{eq:main-bound}
\end{equation}
\end{theorem}

\begin{proof}
The argument follows the PEVI-style decomposition used in MOBILE, combined with a MOReL-style pessimistic MDP
that absorbs on ``unknown'' state-action pairs.

\paragraph{Step 1: Construct a valid (local) uncertainty quantifier.}
Define the penalty function
\begin{equation}
  \Gamma(s,a)\ \doteq\ \beta\,u(s,a)\,\mathbb{I}\!\bigl((s,a)\in\bar{\mathcal{U}}_\delta\bigr)\ +\ V_{\max}\,\mathbb{I}\!\bigl((s,a)\in\mathcal{U}_\delta\bigr).
  \label{eq:Gamma}
\end{equation}
On $\bar{\mathcal{U}}_\delta$, Assumption~\eqref{eq:assump-local} implies
$\bigl|(\hat T^\pi Q)(s,a)-(T^\pi Q)(s,a)\bigr|\le \Gamma(s,a)$.
On $\mathcal{U}_\delta$, we use a trivial bound:
since both $(\hat T^\pi Q)(s,a)$ and $(T^\pi Q)(s,a)$ lie in an interval of length at most $V_{\max}$
(the immediate reward plus a discounted value bounded by $V_{\max}$), we have
$\bigl|(\hat T^\pi Q)(s,a)-(T^\pi Q)(s,a)\bigr|\le V_{\max}=\Gamma(s,a)$.
Therefore $\Gamma$ is a valid Bellman-error upper bound everywhere, while being \emph{data-dependent} (via $u$)
only on $\bar{\mathcal{U}}_\delta$.

\paragraph{Step 2: PEVI/MOBILE-style pessimistic planning bound on the safe region.}
Consider pessimistic planning that evaluates a policy via the operator
$(\hat T^\pi Q)(s,a)-\Gamma(s,a)$.
Standard PEVI analyses (as adopted in MOBILE) imply that the suboptimality of the policy output by pessimistic planning
is bounded by twice the expected cumulative penalty along an optimal trajectory:
\begin{equation}
  J(\pi^\star)-J(\hat\pi)
  \le 2\,\mathbb{E}_{\pi^\star}\!\Bigl[\sum_{t=0}^{H}\gamma^t\,\Gamma(s_t,a_t)\Bigr].
  \label{eq:pevi-style}
\end{equation}
(Informally: the pessimistic value estimates are lower bounds on the true value; the factor $2$ arises from the
telescoping-sum argument bounding $V^{\pi^\star}-V^{\hat\pi}$ by cumulative Bellman estimation gaps, as in PEVI/MOBILE.)

\paragraph{Step 3: Replace the $\mathcal{U}_\delta$ penalty by a hitting-probability term (MOReL-style).}
Let $\tau_{\mathcal{U}}=\inf\{t\ge 0:(s_t,a_t)\in\mathcal{U}_\delta\}$ under $\pi^\star$.
Because MC-ETM truncates rollouts once $\mathcal{U}_\delta$ is encountered, $\mathcal{U}_\delta$ can be visited at most once
along a model rollout; consequently,
\begin{equation}
  \sum_{t=0}^{H}\gamma^t\,\mathbb{I}\!\bigl((s_t,a_t)\in\mathcal{U}_\delta\bigr)
  = \gamma^{\tau_{\mathcal{U}}}\,\mathbb{I}(\tau_{\mathcal{U}}\le H)
  \le \mathbb{I}(\tau_{\mathcal{U}}\le H).
  \label{eq:hit-ineq}
\end{equation}
Thus,
\begin{equation}
  \mathbb{E}_{\pi^\star}\!\Bigl[\sum_{t=0}^{H}\gamma^t\,V_{\max}\,\mathbb{I}\!\bigl((s_t,a_t)\in\mathcal{U}_\delta\bigr)\Bigr]
  \le V_{\max}\,\mathbb{P}_{\pi^\star}\!\bigl(\tau_{\mathcal{U}}\le H\bigr)
  = \frac{R_{\max}}{1-\gamma}\,\mathbb{P}_{\pi^\star}\!\Bigl(\exists t\le H:\ (s_t,a_t)\in\mathcal{U}_\delta\Bigr).
  \label{eq:truncate-term}
\end{equation}
This is the same MOReL-type dependence on the probability of entering the ``unknown'' set (equivalently, the hitting time).

\paragraph{Step 4: Combine terms.}
Plugging $\Gamma$ from \eqref{eq:Gamma} into \eqref{eq:pevi-style} and then using \eqref{eq:truncate-term} yields
\[
J(\pi^\star)-J(\hat\pi)
\le 2\sum_{t=0}^{H}\gamma^t\,\mathbb{E}_{\pi^\star}\!\Bigl[\beta u(s_t,a_t)\mathbb{I}\!\bigl((s_t,a_t)\in\bar{\mathcal{U}}_\delta\bigr)\Bigr]
+ \frac{2R_{\max}}{1-\gamma}\,\mathbb{P}_{\pi^\star}\!\Bigl(\exists t\le H:\ (s_t,a_t)\in\mathcal{U}_\delta\Bigr),
\]
which is exactly \eqref{eq:main-bound}.
\end{proof}

\subsection{Summary on two scenarios}\label{app:two-scenarios}

\paragraph{Scenario I: No OOD encountered.}
If $\mathbb{P}_{\pi^\star}(\exists t\le H:(s_t,a_t)\in\mathcal{U}_\delta)=0$, then the truncation term vanishes and
Theorem~\ref{thm:main} reduces to a pure MOBILE/PEVI-style bound:
\begin{equation}
  J(\pi^\star)-J(\hat\pi)
  \le
  2\sum_{t=0}^{H}\gamma^t\,\mathbb{E}_{\pi^\star}\!\Bigl[\beta\,u(s_t,a_t)\Bigr],
\end{equation}
i.e., the suboptimality is controlled by the (manifold-valid) Model-Bellman inconsistency penalty.

\paragraph{Scenario II: OOD encountered.}
When $\mathbb{P}_{\pi^\star}(\exists t\le H:(s_t,a_t)\in\mathcal{U}_\delta)>0$, we recover the MOReL-type dependence on
the hitting probability of the ``unknown'' set: the second term in \eqref{eq:main-bound} bounds the worst-case loss incurred by
forced termination (absorption) upon hitting $\mathcal{U}_\delta$.
This prevents the bound from depending on potentially unbounded/erratic $u(s,a)$ values in high-energy regions, replacing them
with a uniformly bounded risk term.


\section{Didactic Experiment Setup}

We detail the data generation process and the underlying transition dynamics for the illustrative toy experiment presented in the main text. The environment is designed to evaluate the capacity of various transition models to handle discontinuous, multi-modal, and non-smooth dynamics.

We consider a continuous one-dimensional state space $\mathcal{S} \subset \mathbb{R}$ and a one-dimensional action space $\mathcal{A} \subset \mathbb{R}$. The domain of interest for visualization and evaluation is bounded within the square region $(s, a) \in [-1, 1] \times [-1, 1]$.
The true transition function, $f: \mathcal{S} \times \mathcal{A} \rightarrow \mathcal{S}$, is constructed as a piecewise function composed of a sinusoidal wave, constant plateaus, and sharp discontinuities. The deterministic next state $s'_{clean}$ is governed by the following logic:

\begin{equation}
    s'_0 = f(s, a) = 
    \begin{cases} 
    \sin(-a) + 1.0 & \text{if } |s| < 0.5 \land |a| < 0.5 \\
    1.0 & \text{if } s \ge 0.5 \\
    -1.0 & \text{if } s \le -0.5 \\
    0.0 & \text{otherwise}
    \end{cases}
\label{eq:toy_dynamics}
\end{equation}

This transition function creates discontinuities near the section where either $|s|$ or $|a|$ equals $0.5$, which are particularly challenging for smooth regressors (\emph{e.g.}, standard MLPs) to fit without over-smoothing.

To simulate stochasticity, we add Gaussian noise to the training dataset:
\begin{equation}
    s' = s'_0 + \epsilon, \quad \epsilon \sim \mathcal{N}(0, \sigma^2)
\end{equation}
where $\sigma=0.05$ represents the inherent noise standard deviation.

We construct separate training and validation datasets to simulate the challenge of offline RL, where the training data distribution may not fully cover the evaluation manifold.

\textbf{Training Data Generation:}
The training dataset $\mathcal{D}_{train}$ consists of $N_{train}=1e5$ tuples $(s, a, s')$. The states and actions are sampled independently from a standard normal distribution:
\begin{equation}
    s \sim \mathcal{N}(0, 1), \quad a \sim \mathcal{N}(0, 1)
\end{equation}
This sampling strategy naturally concentrates data near the origin $(0,0)$ and leaves the boundaries of the domain (e.g., corners like $s=1, a=1$) sparsely populated. This setup implicitly tests the model's ability to generalize to valid but low-probability regions.

\textbf{Evaluation Data Generation:}
To visualize the learned energy landscapes and transition surfaces completely, the validation set is generated using a uniform grid over the domain:
\begin{equation}
    s_{val}, a_{val} \in \text{UniformGrid}([-1, 1], [-1, 1])
\end{equation}
This ensures that evaluation metrics cover the sharp discontinuities and plateaus equally, even if they are underrepresented in the training distribution.

\section{Additional Experimental Results}

\label{app:additional_experiments}

We provide a detailed quantitative analysis of MC-ETM based on the ablation studies discussed in Section 5.3 of the main text. Our goal is to empirically validate three core hypotheses: (1) MC-ETM yields superior prediction fidelity compared to standard ETM and MLP regressors; (2) the learned energy function serves as a calibrated proxy for model error; and (3) both the manifold constraint and the hybrid truncation-penalization mechanism are essential for robust offline policy optimization.

\subsection{Dynamics Model Prediction Error}
To assess the quality of the learned dynamics independent of the policy, we evaluate the prediction accuracy of MC-ETM against baseline transition models. We report the error as the absolute value between the predicted and ground truth next states across various environments. To further evaluate the generalizability of the prediction model, we cross-verify these models on datasets collected by different behavior policies within the same environment. The results are shown in Table~\ref{tab:cross}, with the lowest results marked in bold. Among them, MC-ETM utilizes the low-dimensional manifold structure of the datasets to further reduce the prediction error compared with standard ETM, highlighting the importance of the MPD mechanism.

\subsection{Energy-Error Correlation Analysis}
A critical premise of our method is that the learned energy $E_\theta(s,a,s')$ is a valid signal for Out-of-Distribution (OOD) detection. High energy should correspond to a high prediction error. To verify this, we compute the Pearson Correlation Coefficient (PCC) between the terminal energy value after inference and the ground-truth prediction error on a held-out test set in the halfcheetah-medium dataset, which contains both in-distribution samples and OOD samples injected with noise.

\begin{figure}[h]
    \centering
    \includegraphics[width=0.9\textwidth]{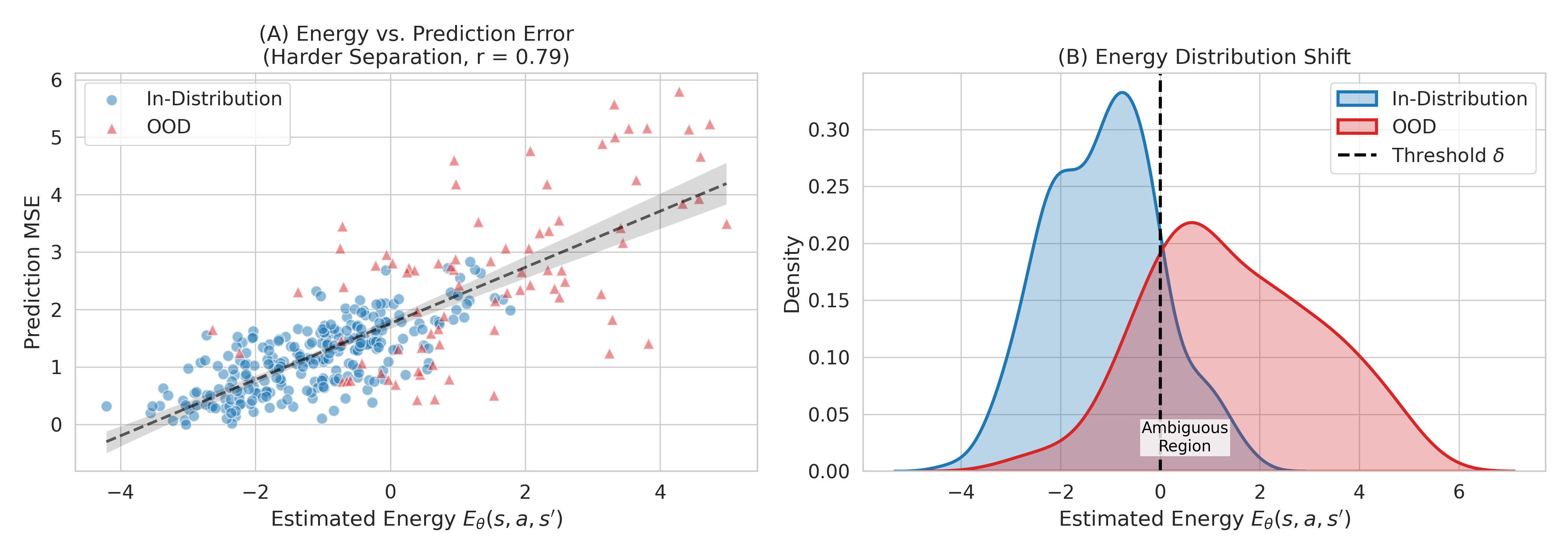}
    \caption{Correlation between MC-ETM energy values and prediction MSE. (Left) Scatter plot showing a positive correlation; higher energy implies higher model error. (Right) Histogram of energy values for In-Distribution (ID) vs. Out-of-Distribution (OOD) samples, showing a clear separation that enables effective truncation. The threshold $\delta$ is set to 0 in this case.}
    \label{fig:energy_correlation}
\end{figure}

The results, visualized in Figure~\ref{fig:energy_correlation}, demonstrate a strong positive correlation ($r = 0.79$) between energy and error. Furthermore, the energy distribution of OOD samples is distinct from that of ID samples, validating the use of the hard energy threshold $\delta$ in Eq.~(10) to truncate unreliable rollouts.

\subsection{Ablation Studies on Policy Optimization}
We dissect the contribution of each component in the MC-ETM policy optimization framework. We compare the full MC-ETM algorithm against three ablations on medium datasets from the hopper, halfcheetah and walker2d environments in D4RL MuJoCo tasks.:
\begin{enumerate}
    \item \textbf{No Manifold Constraint (w/o MPD):} Replaces MC-ETM with a standard ETM trained using Gaussian noise negatives.
    \item \textbf{No Truncation (w/o $\nu$):} Removes the hard rollout truncation, relying solely on the soft uncertainty penalty.
    \item \textbf{No Penalty (w/o $\lambda$):} Sets $\lambda=0$, removing the epistemic uncertainty penalty and relying solely on truncation.
\end{enumerate}

\begin{table}[h]
    \centering
    \caption{Ablation study normalized returns on medium datasets from the hopper, halfcheetah and walker2d environments in D4RL MuJoCo tasks. The degradation in performance across all ablations confirms that both accurate manifold modeling and the hybrid pessimistic mechanism are necessary.}
    \label{tab:ablation_results}
    \begin{tabular}{lcccc}
        \toprule
        \textbf{Task} & \textbf{Full MC-ETM} & \textbf{w/o MPD} & \textbf{w/o Truncation} & \textbf{w/o Penalty} \\
        \midrule
        hopper-m       & $\mathbf{107.0} \pm 1.1$ & $98.7 \pm 3.9$ & $93.5 \pm 5.9$ & $69.3 \pm 21.1$ \\
        walker2d-m     & $\mathbf{92.7} \pm 0.7$ & $86.7 \pm 1.1$ & $89.5 \pm 2.3$ & $66.0 \pm 11.7$ \\
        halfcheetah-m  & $\mathbf{76.9} \pm 0.6$ & $70.4 \pm 2.1$ & $73.6 \pm 1.7$ & $54.9 \pm 5.3$ \\
        \bottomrule
    \end{tabular}
\end{table}

The results are shown in Table ~\ref{tab:ablation_results}, from which we can conclude that removing the penalty process causes a significant drop in performance. Even with truncation, the OOD samples will not be completely removed; the policy will still exploit the high prediction error in those ambiguous regions, resulting in compromised performance. Additionally, removing MPD is also harmful to training robustness because the standard ETM's energy landscape is too "loose" to be effectively utilized in predicting dynamics or detecting OOD samples. Besides, the results also demonstrate that hard truncation acts as a necessary safety gate, while relying only on soft penalization ($\lambda$) fails in regions far from the support, where the ensemble variance may not accurately reflect the true error magnitude. 

\subsection{Hyperparameter Sensitivity}
Finally, we analyze the sensitivity of MC-ETM to three key hyperparameters: the energy threshold $\delta$, the penalty coefficient $\lambda$, and the latent manifold dimension $d_z$.

\textbf{Energy Threshold ($\delta$):}
The threshold $\delta$ controls the size of the ``trust region.''
\begin{itemize}
    \item \textit{Low $\delta$:} The agent is overly conservative, truncating rollouts too early and behaving like a behavior cloning agent.
    \item \textit{High $\delta$:} The agent explores OOD regions where model error dominates.
    \item \textit{Optimal Range:} We find that setting $\delta$ based on the $95^{th}$ percentile of energy values in the training set provides a robust default across tasks.
\end{itemize}

\textbf{Latent Dimension ($d_m$):}
The dimension of the manifold $d_m$ balances reconstruction fidelity and compactness.
\begin{itemize}
    \item If $d_m$ is too small (e.g., $d_m=2$ for high-DoF robots), the autoencoder fails to reconstruct valid transitions ($s' \approx f_d(f_e(s'))$ degrades), creating a "bottleneck" error.
    \item If $d_m$ is too large (approaching the ambient dimension $D$), the MPD mechanism loses its efficacy because the "manifold" volume becomes sparse, and the curse of dimensionality returns.
    \item Empirically, we set $d_m$ to 5 for the hopper and 10
    for the halfcheetah and walker2d, which work best.
\end{itemize}

\textbf{Penalty Coefficient ($\lambda$):}
MC-ETM's sensitivity to $\lambda$ is mostly consistent with the prior uncertainty-based method MOBILE, while we made a slight adjustment to select the best $\lambda$ from $[0.5, 2.5]$. Basically, increasing $\lambda$ from $0$ generally yields more stable results until $\lambda$ is large enough to stifle all policy improvement.

\begin{table}[h]
\centering
\caption{Cross-environment Comparison of Mean Prediction Error between MLP, standard ETM and MC-ETM}
\label{tab:cross}
\begin{tabular}{@{}llccc@{}}
\toprule
\textbf{Model Env} & \textbf{Test Env} & \textbf{MLP} & \textbf{standard ETM} & \textbf{MC-ETM} \\ \midrule
\textbf{Hopper} &  &  &  & \\ 
random & random & $0.005827$ & $0.007405$ & $\mathbf{0.001986}$ \\ 
random & medium & $0.188629$ & $0.131319$ & $\mathbf{0.078207}$ \\ 
random & medium-replay & $0.147348$ & $0.123200$ & $\mathbf{0.066899}$ \\ 
random & medium-expert & $0.185224$ & $0.136078$ & $\mathbf{0.073670}$ \\ 
medium & random & $0.062951$ & $0.121988$ & $\mathbf{0.036943}$ \\ 
medium & medium & $0.011889$ & $0.009781$ & $\mathbf{0.003919}$ \\ 
medium & medium-replay &$0.056291$ & $0.044484$ & $\mathbf{0.032487}$ \\ 
medium & medium-expert & $0.020411$ & $0.018276$ & $\mathbf{0.007986}$ \\ 
medium-replay & random & $0.033892$ & $0.018150$ & $\mathbf{0.006980}$ \\ 
medium-replay & medium & $0.021166$ & $0.009098$ & $\mathbf{0.006313}$ \\ 
medium-replay & medium-replay & $0.025752$ & $0.008759$ & $\mathbf{0.006485}$ \\ 
medium-replay & medium-expert & $0.023580$ & $0.009679$ & $\mathbf{0.006846}$ \\ 
medium-expert & random & $0.055056$ & $0.042358$ & $\mathbf{0.038354}$ \\ 
medium-expert & medium & $0.010442$ & $0.009564$ & $\mathbf{0.003640}$ \\ 
medium-expert & medium-replay & $0.051497$ & $0.042419$ & $\mathbf{0.033011}$ \\ 
medium-expert & medium-expert & $0.009301$ & $\mathbf{0.002014}$ & $0.003147$ \\ \midrule
\textbf{HalfCheetah} &  &  &  & \\ 
random & random & $0.186056$ & $0.120956$ & $\mathbf{0.120095}$ \\ 
random & medium & $0.662978$ & $0.599604$ & $\mathbf{0.465097}$ \\ 
random & medium-replay & $0.555700$ & $0.443773$ & $\mathbf{0.364296}$ \\ 
random & medium-expert & $\mathbf{1.118878}$ & $1.319108$ & $1.141470$ \\ 
medium & random & $0.593238$ & $0.349435$ & $\mathbf{0.298135}$ \\ 
medium & medium & $0.194258$ & $0.246826$ & $\mathbf{0.093622}$ \\ 
medium & medium-replay & $0.388058$ & $0.237622$ & $\mathbf{0.193593}$ \\ 
medium & medium-expert & $0.393646$ & $1.795844$ & $\mathbf{0.343527}$ \\ 
medium-replay & random & $0.459957$ & $0.254801$ & $\mathbf{0.225953}$ \\ 
medium-replay & medium & $0.313714$ & $0.352766$ & $\mathbf{0.148552}$ \\ 
medium-replay & medium-replay & $0.357627$ & $0.295647$ & $\mathbf{0.164714}$ \\ 
medium-replay & medium-expert & $0.604161$ & $\mathbf{0.314497}$ & $0.441931$ \\ 
medium-expert & random & $0.554667$ & $0.491406$ & $\mathbf{0.320083}$ \\ 
medium-expert & medium & $0.208558$ & $0.153605$ & $\mathbf{0.104615}$ \\ 
medium-expert & medium-replay & $0.380486$ & $0.275669$ & $\mathbf{0.206782}$ \\ 
medium-expert & medium-expert & $0.166779$ & $0.159534$ & $\mathbf{0.087272}$ \\ \midrule
\textbf{Walker2d} &  &  &  & \\ 
random & random & $0.237939$ & $0.194076$ & $\mathbf{0.167385}$ \\ 
random & medium & $1.464581$ & $\mathbf{1.083580}$ & $1.759252$ \\ 
random & medium-replay &$1.333978$ & $\mathbf{0.925014}$ & $1.194890$ \\ 
random & medium-expert & $1.692117$& $\mathbf{1.168266}$ & $2.010816$ \\ 
medium & random & $0.996625$ & $0.862684$ & $\mathbf{0.544237}$ \\ 
medium & medium & $0.163006$ & $0.302731$ & $\mathbf{0.056605}$ \\ 
medium & medium-replay & $0.395338$ & $0.410104$ & $\mathbf{0.176679}$ \\ 
medium & medium-expert & $0.221335$ & $0.118636$ & $\mathbf{0.084960}$ \\ 
medium-replay & random & $0.809099$ & $0.790565$ & $\mathbf{0.347924}$ \\ 
medium-replay & medium & $0.216737$ & $0.134870$ & $\mathbf{0.102339}$ \\ 
medium-replay & medium-replay & $0.298443$ & $0.203577$ & $\mathbf{0.131649}$ \\ 
medium-replay & medium-expert &$0.245375$ & $0.236785$ & $\mathbf{0.115834}$ \\ 
medium-expert & random & $0.867276$ & $0.875764$ & $\mathbf{0.576945}$ \\ 
medium-expert & medium & $0.115061$ & $0.093144$ & $\mathbf{0.056529}$ \\ 
medium-expert & medium-replay & $0.340651$ & $0.407487$ & $\mathbf{0.175225}$ \\ 
medium-expert & medium-expert & $0.093424$ & $0.084948$ & $\mathbf{0.045204}$ \\ \bottomrule
\end{tabular}
\end{table}
\section{Rollout Visualization}

We demonstrate the visualization of simulated rollouts in Figure \ref{fig:all_envs}. The results of the ground truth, MC-ETM, standard ETM, and MLP-based dynamic models are listed in order from top to bottom. Each model is evaluated starting from an initial state and follows the same action sequence in the offline dataset. The results indicate that MC-ETMs have a much stronger capability in simulating dynamics and reducing cumulative errors, especially for the Mujoco environment involving discontinuous physical rigid-body transitions.

\begin{figure}
    \centering
    \begin{subfigure}{1\linewidth}
        \centering
        \includegraphics[width=1\linewidth]{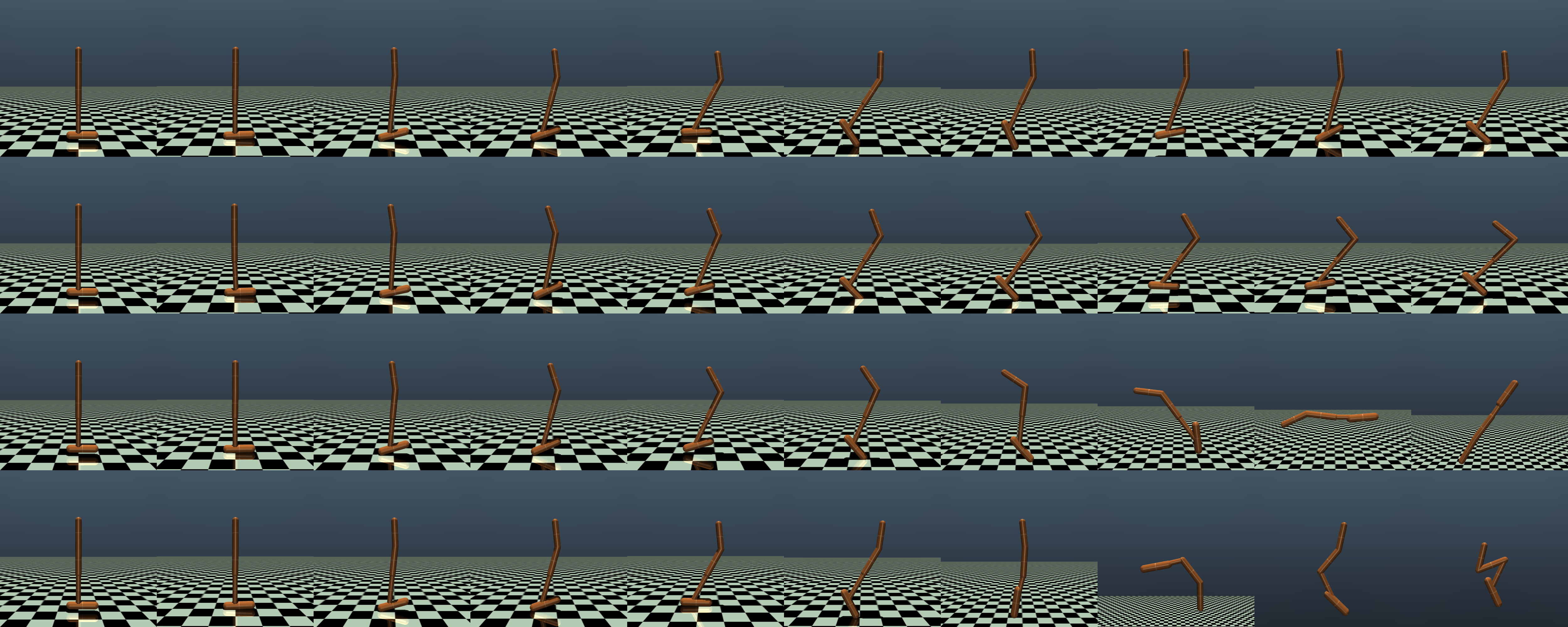}
        \caption{hopper-medium-expert}
        \label{fig:hopper}
    \end{subfigure}\\[1ex]
    \begin{subfigure}{1\linewidth}
        \centering
        \includegraphics[width=1\linewidth]{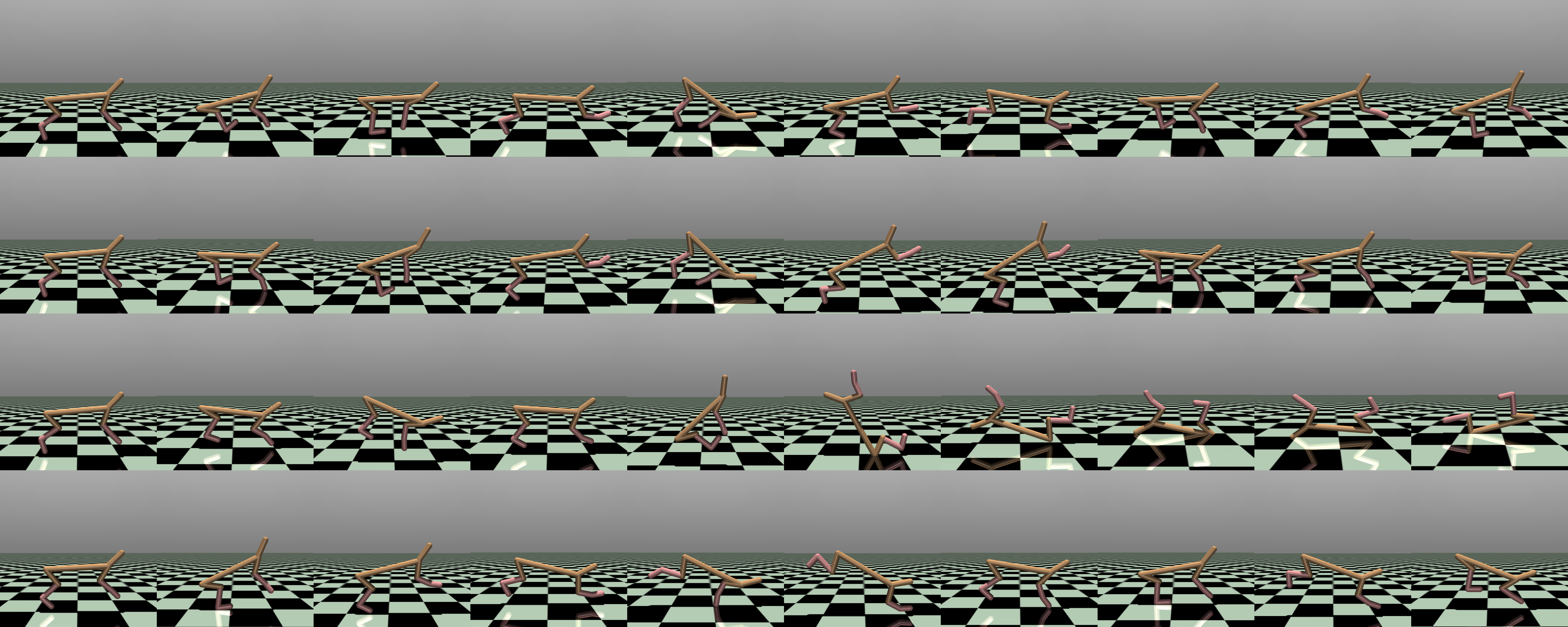}
        \caption{halfcheetah-medium-expert}
        \label{fig:halfcheetah}
    \end{subfigure}\\[1ex]
    \begin{subfigure}{1\linewidth}
        \centering
        \includegraphics[width=1\linewidth]{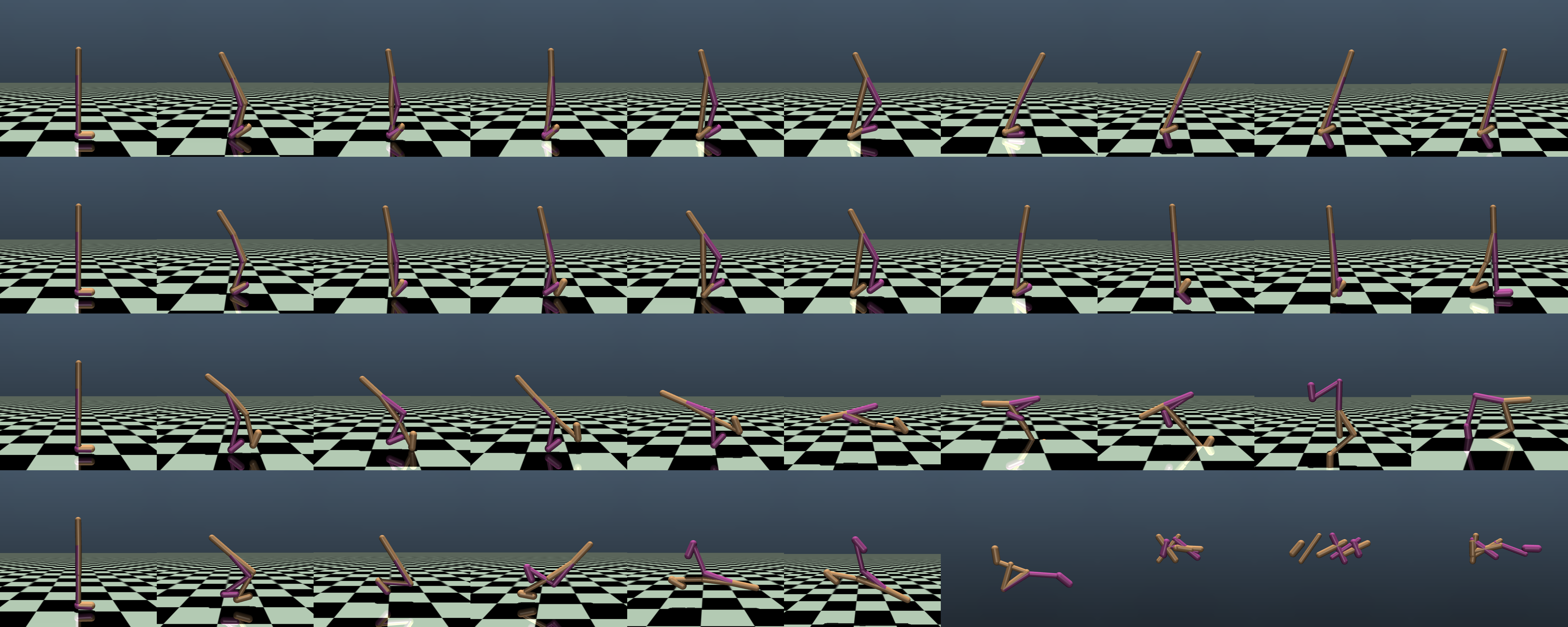}
        \caption{walker2d-medium-expert}
        \label{fig:walker2d}
    \end{subfigure}
    \caption{Visualized rollout trajectories with ground truth, MC-ETM, standard ETM and MLP in different Mujoco tasks.}
    \label{fig:all_envs}
\end{figure}

\end{document}